\documentclass[letterpaper, 10 pt, conference]{ieeeconf} 
\IEEEoverridecommandlockouts     

\usepackage[ruled,vlined,linesnumbered]{algorithm2e}
\usepackage[cmex10]{amsmath}
\usepackage{amssymb,mathtools}
\usepackage{color}
\usepackage{graphicx}
\usepackage{wrapfig}
\usepackage{verbatim} 
\usepackage{enumerate}
\usepackage{pifont}
\usepackage{hyperref}

\newcommand{\sq}{\hbox{\rlap{$\sqcap$}$\sqcup$}}
\newcommand{\qed}{\hspace*{\fill}\sq}

\newtheorem{theorem}{Theorem}

\newtheorem{lemma}{Lemma}
\newtheorem{definition}{Definition}

\newcommand{\cR}{{\cal Q}}

\newcommand{\dist}{{\rm dist}}

\newcommand{\cov}{{\sc OnlineCPP}}
\newcommand{\cof}{{\sc OfflineCPP}}

\newcommand{\cova}{{\sc OnlineCPPAlg}}
\newcommand{\change}[1]{\textcolor{red}{#1}}

\begin{document}
\title{\LARGE \bf A Constant-Factor Approximation Algorithm for Online \\Coverage Path Planning with Energy Constraint}

\author{Ayan Dutta, Gokarna Sharma
\thanks{A. Dutta is with the School of Computing, University of North Florida, Jacksonville, FL 32224, USA
{\tt\small a.dutta@unf.edu}} \thanks{G. Sharma is with the Department of Computer Science,
Kent State University, Kent, OH 44242, USA
{\tt\small sharma@cs.kent.edu}
}
}

\maketitle
\thispagestyle{empty}
\pagestyle{empty}

\begin{abstract}
In this paper, we study the problem of coverage planning by a mobile robot with a limited energy budget. The objective of the robot is to cover every point in the environment while minimizing the travelled path length. The environment is initially unknown to the robot. Therefore, it needs to avoid the obstacles in the environment on-the-fly during the exploration. As the robot has a specific energy budget, it might not be able to cover the complete environment in one traversal. Instead, it will need to visit a static charging station periodically in order to recharge its energy. To solve the stated problem, we propose a budgeted depth first search (DFS)-based exploration strategy that helps the robot to cover any unknown planar environment while bounding the maximum path length to a constant-factor of the shortest-possible path length. Our $O(1)$-approximation guarantee advances the state-of-the-art of log-approximation for this problem. 
Simulation results show that our proposed algorithm outperforms the current state-of-the-art algorithm both in terms of the travelled path length and run time in all the tested environments with concave and convex obstacles.
\end{abstract}

\section{Introduction}
\label{section:introduction}
Coverage planning is the task of finding a path or a set of paths to cover all the points in an environment \cite{choset2001coverage}. In robotics, this problem has many potential real-world applications including autonomous sweeping, vacuum cleaning, and lawn mowing. In an {\em online} version of the problem, the area of interest is initially unknown to the robot. Therefore it needs to discover and avoid the unknown obstacles in the environment while covering all the points in the free space by traveling as minimum distance as possible \cite{choset2001coverage}. 

Traditionally, this problem has been studied assuming that the robot has an unlimited energy budget, where given a robot, a single
path can be planned to cover the given environment. 
The  {\em offline} version of the problem where the robot(s) have a priori knowledge of the environment including obstacles has been well studied, e.g., see \cite{Galceran:2013}. Many algorithms have been proposed such as the boustrophedon decomposition
based coverage \cite{Choset:2000,Mannadiar2010}, the spiral path coverage
\cite{Gonzlez2005}, and the spanning-tree based coverage
\cite{Gabriely:2001}. 
These techniques can also be adapted to solve {\em online} coverage planning \cite{WeiI18,ChoiLBO09}. 

In practice however, the robots do not have unlimited energy available. Therefore, even covering a standard-size environment (e.g., a farm) while simultaneously using on-board sensors (e.g., camera) becomes prohibitive with a single charge. A battery-powered robot needs to return to the charging station to get recharged before the battery runs out. Due to practical relevance, in the recent years, there has been a significant volume of work on the {\em energy-constrained} coverage planning problem  \cite{Strimel2014, MishraRMA16, WeiI18, WeiICRA18, ShnapsR16,Sharma2019}. 
The offline version of the problem, denoted as {\cof}, is studied in  \cite{ShnapsR16,WeiI18, WeiICRA18} and the online version, denoted as {\cov}, is studied in \cite{Sharma2019,ShnapsR16}. 
The state-of-the-art algorithm for {\cof} is due to \cite{WeiICRA18}, which provides $O(1)$-approximation. For {\cov}, the state-of-the-art algorithm is due to \cite{Sharma2019}, which provides  $O(\log (B/L))$-approximation, where $B$ is the energy budget and $L$ is the size of the robot (assuming a $L \times L$ square robot). 
Our goal in this paper is to provide a better approximation algorithm for {\cov}, i.e., to reduce the $O(\log (B/L))$-approximation of \cite{Sharma2019} to $O(1)$. 
This will show that asymptotically there is no approximation gap between {\cov} and {\cof}. 

Our proposed algorithm covers an unknown environment through a DFS traversal approach tailored for the limited energy budget $B$. Robot $r$ performs a {\em depth first search} traversal while building a tree map of the environment on-the-fly. It returns to the charging station to get its battery fully recharged (stopping the DFS exploration) when the path length of the traversal becomes at most $B$. 
After the battery is fully charged, $r$ then moves to the cell where it stopped the DFS process, and continues traversing $P$.  
Simulation results show that our proposed algorithm is up to $5.80$ times faster and up to $2.53$ times less costlier (in terms of traversed path length) than the state-of-the-art algorithm \cite{Sharma2019}.

\noindent{\bf Contributions.}
Initially, the robot is at the charging station $S$ that is inside $P$.
The goal of {\cov} is to find a set of paths $\cR=\{\cR_1,\ldots,\cR_k\}$ for the robot such that 
\begin{itemize}
\item {\bf Condition (a):} Each path $\cR_i$ starts and ends at $S$ 
\item {\bf Condition (b):} Each path $\cR_i$ has length $l(\cR_i)\leq B$ 
\item {\bf Condition (c):} 
The paths in $\cR$ collectively cover the environment $P$, i.e., $\cup_{i=1}^{n}\cR_i=P$, 
\end{itemize}

and the following two performance metrics are optimized:  
\begin{itemize}
\item {\bf Performance metric 1:} The {\em number of paths} 
in $\cR$, denoted as $|\cR|$, is minimized, and
\item {\bf Performance metric 2:} The {\em total lengths of the paths} in $\cR$, denoted as $l(\cR)=\sum_{i=1}^{n} l(\cR_i)$, is minimized. 
\end{itemize}

We establish the following main theorem for {\cov}. 

\begin{theorem}[{\bf Main Result}]
\label{theorem:main}
Given an unknown planar polygonal environment $P$ possibly containing obstacles and a robot $r$ of size $L\times L$ consisting of position and obstacle detection sensors initially situated at a charging station $S$ inside $P$ with energy budget $B$, there is an algorithm that correctly solves {\cov} and guarantees 
$10$-approximation to both performance metrics compared to the optimal algorithm that has complete knowledge about $P$.  
\end{theorem}


This result clearly advances the current state-of-the-art as it improves upon the log-approximation provided in \cite{Sharma2019} and provides a constant-factor approximation. 
Furthermore, the proposed algorithm is easier to implement than \cite{Sharma2019}. 

\noindent{\bf Related Work.}
The most closely related works to ours are 
\cite{Sharma2019,WeiI18,WeiICRA18,ShnapsR16}.  
Shnaps and Rimon \cite{ShnapsR16}  proposed an $1/(1-\rho)$-approximation algorithm for {\cof}, where $\rho$ is the ratio between the furthest distance between any two cells in the environment
and half of the energy budget \cite{ShnapsR16}. 
For {\cov}, they proposed an
$O(B/L)$-approximation algorithm. 
Wei and Isler \cite{WeiI18} presented an
$O(\log(B/L))$-approximation algorithm for {\cof}, which has been improved to a constant-factor approximation by them in \cite{WeiICRA18}. 
Recently, Sharma {\it et al.} \cite{Sharma2019} provided an $O(\log (B/L))$-approximation algorithm for {\cov}. 
In this paper, we improve upon the log-approximation bound and provide the $O(1)$-approximation to {\cov}. 

The other related work is the coverage of a graph. The goal is to design paths to visit every vertex of the given graph. Without energy constraints, it becomes the well-known {\em Traveling Salesperson Problem} (TSP) \cite{Applegate:2007} and a DFS traversal provides a constant-approximation of the TSP. With energy constraints, this coverage problem becomes the {\em Vehicle Routing
Problem} (VRP) \cite{Laporte92}. One version of VRP is the {\em Distance} Vehicle Routing Problem (DVRP), which models the energy consumption proportional to the distance travelled.
For DVPR on tree metrics, Nagarajan and Ravi \cite{Nagarajan:2012}
proposed a 2-approximation algorithm. Li {\it et al.} \cite{Li:1992} used a TSP-partition method and their algorithm has a similar approximation to the work in \cite{ShnapsR16}. Most of these work studied the offline version 
so that pre-processing on the environment can be done prior to exploration obtaining better approximation. This is also the case in the algorithm of \cite{WeiI18,WeiICRA18} for {\cof}. 
Coverage with multiple robots has also received a lot of attention (e.g., see \cite{Brass:2009,Fraigniaud:2006}). In some cases, the paths planned for a single robot
under energy constraints can also be executed by multiple robots by assigning the planned paths to the robots, without affecting the total cost. In this paper, we consider coverage planning with a single robot.

\section{Problem setup}
\label{section:model}
In this paper, we use the same model as in \cite{Sharma2019,WeiI18,ShnapsR16}.

\vspace{1mm}
\noindent{\bf Environment.} The environment $P$ is a planar polygon containing a single charging station $S$ inside it. 
$P$ may possibly contain polygonal, static obstacles. 
See left of Fig.~\ref{fig:tree} for an illustration of $P$ with an obstacle $O_1$. 
The environment $P$ is discretized into cells forming a 4-connected grid. 

\vspace{1mm}
\noindent{\bf Robot.} We consider the robot $r$ to be initially positioned at the charging station $S$. $r$ has size $L\times L$ that it fits within a grid-cell in $P$. 
The robot $r$ moves rectilinearly in $P$, i.e., it may move to any of the four neighbor cells (if the cell is not occupied by an obstacle) from its current cell.  
We also assume that $r$ has the knowledge of the global coordinate system through a compass on-board, that means it knows left (West), right (East), up (North), and down (South) cells consistently from its current cell. 
Robot $r$ is equipped with a position sensor (e.g., GPS) and an obstacle-detection sensor (e.g., laser rangefinder). We assume that with the laser rangefinder, the robot can detect obstacles in any of its neighbor cells. 
The robot has sufficient on-board memory to store information necessary to facilitate the coverage process. 
Moreover, we assume that initially $r$ does not have any knowledge about $P$, i.e., $P$ is an unknown environment. For the feasibility of covering all cells of $P$, we assume that $P$ is as big as a circle of radius $\lfloor B/2\rfloor$ with center at $S$. It is assumed that the energy consumption of the robot is proportional to the distance travelled, i.e., the energy budget of $B$ allows the robot to move $B$ units distance.

A path (route) $\cR_i$ is a list of cells that $r$ visits starting and ending with $S$. 
Notice that if there are some obstacles within $P$ located in such a way that they divide $P$ into two sub-polygons $P_1$ and $P_2$ with $P_1$ and $P_2$ sharing no common boundary, then $r$ cannot fully cover $P$. Therefore, we assume that there is no such cell $c$ in $P$. That means, there is (at least) a route from $S$ to any obstacle-free cell of $P$. 

We call a cell {\em free} if it is not occupied by an obstacle. We call a cell {\em reachable} if it satisfies the definition below. 

\begin{definition}[Reachable Cell]
\label{definition:reachable}
Any cell $c$ in $P$ is called {\em  reachable} by the robot $r$, if and only if 
(a) it is a free cell,
(b) it is within distance $\lfloor B/2\rfloor$ from $S$, and 
(c) there must be at least a route of consecutive free cells from $S$ to $c$.
\end{definition}

\vspace{1mm}
\noindent{\bf 
{\cov}.} 
The problem is formally defined as follows.
\begin{definition}
Given an unknown planar polygonal environment $P$ possibly containing obstacles with a robot $r$ having battery budget of $B$ initially positioned at a charging station $S$ inside $P$, {\cov} is for $r$ to visit all the reachable cells of $P$  through a set of paths  so that 
\begin{itemize}
\item Conditions (a)--(c) are satisfied, and 
\item Performance metrics (1) and (2) are minimized.
\end{itemize}
\end{definition}

Following \cite{Sharma2019,WeiICRA18}, we measure the efficiency of any algorithm for {\cov} in terms of {\em approximation ratio} which is the worst-case ratio of the cost of the online algorithm for some environment $P$ over the cost of the optimal, offline algorithm for the same environment. 

\section{Processing Unknown Environment}
\label{section:basic}

In this section, we discuss how the robot decomposes the environment into square grid cells and then construct a tree map of the environment on-the-fly.

\vspace{1mm}
\noindent{\bf Decomposition of the Environment.} 
\label{subsection:decompose}
Following \cite{Sharma2019,WeiICRA18}, we decompose the environment $P$ into square cells of size $L \times L$, which is the size of the robot itself. 
\begin{figure}
\centering
\includegraphics[width=2.40in]{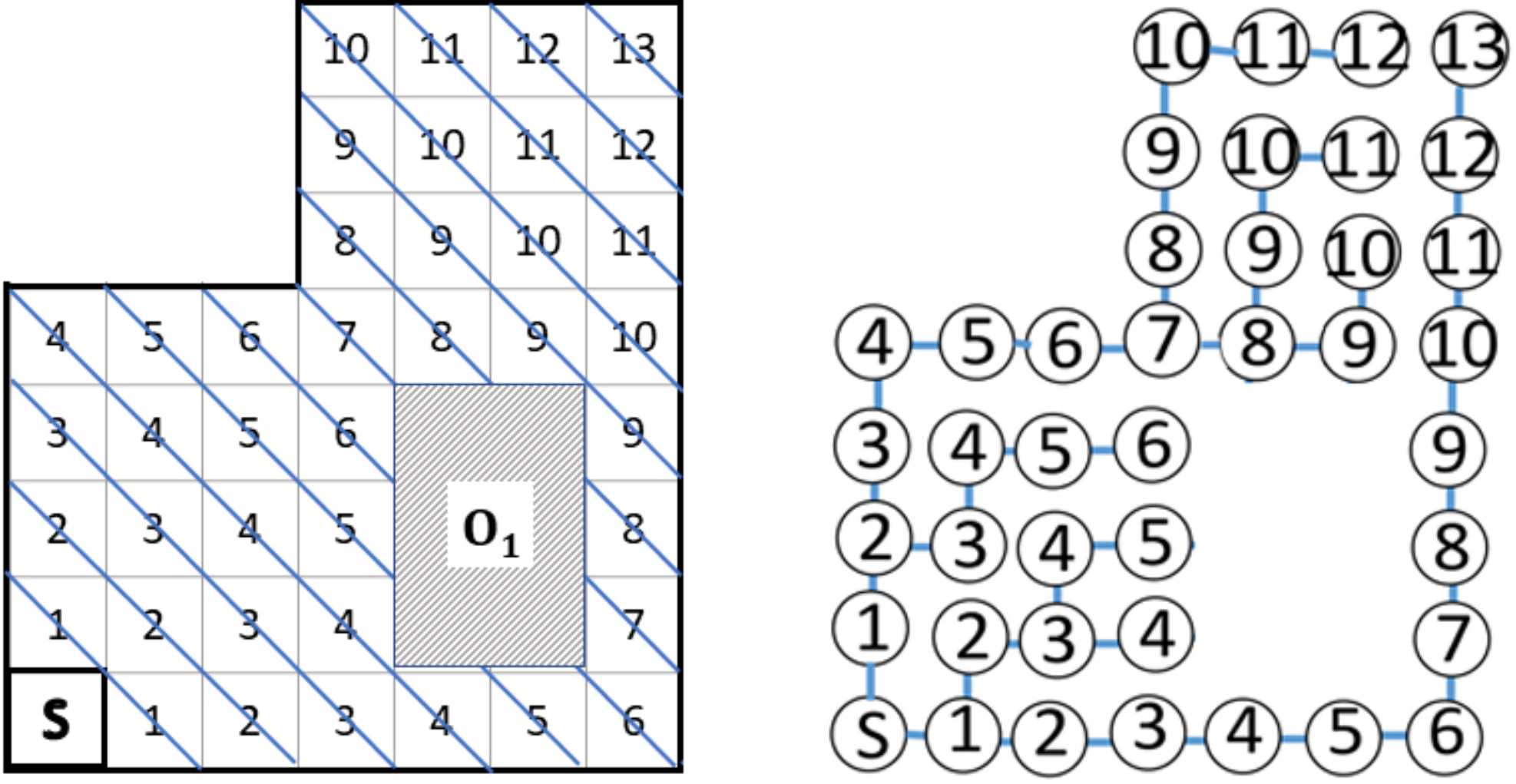}
\caption{({\bf left}) An environment $P$ with an obstacle $O_1$  and a charging station $S$. 
$P$ is shown decomposed as cells of size $L\times L$ same as the robot; ({\bf right}) An example tree map $T_P$ (right) constructed for the environment $P$ on the left. The cells at any contour $C_d$ are at depth $d$ in $T_P$. 
}
\label{fig:tree}
\end{figure}
An {\em equi-distance contour} is a poly-line where the cells on it has the same distance to/from the base point $S$ (the left of Fig.~\ref{fig:tree}). The cells on a contour can be ordered from one side to the other.

Let $c$ be a cell and $C$ be a contour. Let $d(c)$ denote the distance to $S$ from $c$ and let $d(C)$ denote the distance to $S$ from $C$. If $d(C_j)=d(C_i)+1$, we say that contour $C_j$ is contour $C_i$'s next contour. The contour $C$ with $d(C)=1$ is called the first contour. Each cell in a contour has at most $4$ reachable cells from $S$ that are its neighbors.

\vspace{1mm}
\noindent{\bf Constructing a Tree Map.}
\label{section:tree}
Initially, the robot $r$ is placed at the fixed charging station $S$. In this case, the tree, denoted by $T_{P}$, has only one node $S$, which we call the {\em root} of $T_{P}$. If there is no obstacle in $P$, each cell except the boundary cells of $P$ will have exactly four neighbor cells. 

Robot $r$ picks the first free cell $c_1$ according to a clockwise ordering of its neighbors starting from the west and ending in the south neighbor cell. $r$ then inserts it into $T_P$ as a {\em child} of $S$. If the cell labeled West is a reachable cell, then $r$ picks that cell. Otherwise, it goes in order of North, East, and South until it finds the first cell that is reachable. 
We now have two nodes in $T_P=\{S,c_1\}$, with $c_1$ as a child node of $S$.  Furthermore, $c_1$ is a cell in the first contour $C_1$. 

Since $r$ is building $T_P$ while exploring $P$, it will move to $c_1$ after it is included as a child in $T_P$. The robot $r$ then again repeats the process of building $T_P$ from its current cell $c_1$. While at $c_1$, $r$ is only allowed to add one of the neighboring cells of $c_1$ that are in the second contour $C_2$ (i.e., $d(C_2)=2$) as a child of $c_1$. For this, $r$ will include a neighboring cell $c_{2}$ of $c_1$ in $T_P$ only if $c_{2}$ is a cell in contour $C_2$. 
Furthermore, if some cell is already a part of $T_P$, then this cell will not be included in $T_P$ again. 
This process will then continue. The right of Fig.~\ref{fig:tree} provides an illustration of the tree map $T_P$ developed for the environment $P$ shown on the left.  

Essentially, any edge of $T_P$ connects two cells $c_i,c_{i+1}$ of $P$ such that $c_i\in C_i$ and $c_{i+1}\in C_{i+1}$, for $0\leq i\leq \lfloor B/2L\rfloor-1$; in Fig.~\ref{fig:tree}, each cell of contour $C_i$ on the environment $P$ on the left are at depth $i$ in $T_P$ shown on the right.  
Therefore, using this approach, all the cells in the first contour $C_1$  will be children of $S$ (the root of $T_P$), all the cells in the second contour $C_2$ will be children of the nodes of $T_P$ that are cells in the first contour $C_1$, and so on. 

There is one potential problem in certain situations. Consider the environment shown in  Fig.~\ref{fig:changecontour} where the horizontal line passing through $S$ is crossing obstacle $O_1$. The contour numbering and constructing $T_P$ based on contour numbers do not work as the cells in the right of $O_1$ may not be visited by $r$ following Algorithm \ref{algorithm:onlinecppalg} as it requires the robot to visit the cells in an increasing order of contour numbers.  
This is because the contour numbers for those cells are smaller than the contour number of the cells on North, South, and West of the obstacle. For example, see the left of Fig. \ref{fig:changecontour}.   

\begin{figure}[!t]
\centering
\includegraphics[height=1.05in]{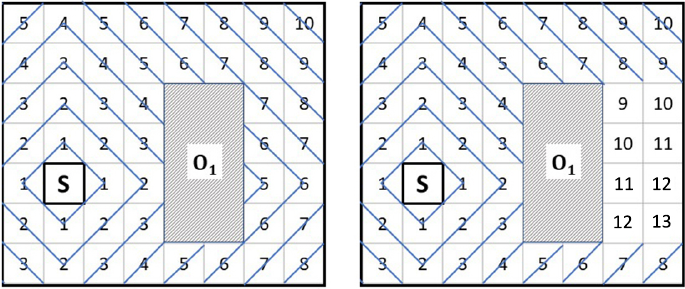}
\vspace{-3mm}
\caption{An illustration of changes on contour numbers}
\label{fig:changecontour}
\end{figure}
We solve this problem in Algorithm \ref{algorithm:onlinecppalg} using an approach where the robot $r$ detects this problem and changes the contour numbers of those cells on-the-fly. See for example the  right of Fig. \ref{fig:changecontour} that depicts how the contour numbers for the cells on the right of $O_1$ are updated on-the-fly by the robot. Details are omitted on how $r$ detects the problem and solves it 
due to space constraints. 

\section{Algorithm}
\label{section:algorithm}
In the description of the algorithm and its analysis, we assume that $L=1$. A simple adaptation will work when $L>1$.
The pseudocode is given in Algorithm \ref{algorithm:onlinecppalg} and illustration of the working principle of the algorithm is given in Fig.~\ref{fig:paths}. 

\subsection{A Naive Approach without Energy Constraint}
The main idea behind our algorithm is to let $r$ incrementally explore the environment $P$ while simultaneously constructing a tree map $T_P$ of $P$ to keep track of the new frontiers that need to be visited by it. 

For simplicity, let us first consider that $T_P$ (or $P$) is known a priori and $r$ has no energy constraint (i.e., $B=\infty$). Let $\cR_{\infty}(r)$ be a route in $T$ that visits all the nodes of $T_P$,  obtained performing a {\em Depth First Search} (DFS) traversal of $T_P$. 
Since $T_P$ is a tree, it is known that all the nodes of $T_P$ can be covered by 
the DFS traversal by visiting each node of $T_P$ at most twice. Therefore, if there are $n$ nodes in $T_P$, then the length of the route $l(\cR(r))\leq 2n$. 
Moreover, an optimal algorithm for $r$ to traverse all $n$ nodes of $T_P$ must have length $l(Q_{OPT}(r))\geq n$, since $r$ can only visit the nodes of $T_P$ sequentially one after another using any algorithm. 
Therefore, without any energy constraint ($B=\infty$), we have a $2$-approximation algorithm. 
The 2-approximation can also be guaranteed for {\cov} when $B=\infty$ since with the knowledge of the global coordinate system, $r$ can visit all the nodes of $T_P$ as if $T_P$ is known a priori, satisfying  the length of the route $l(\cR(r))\leq 2n$.   

\subsection{Incorporating the Energy Constraint}
Now suppose that $r$ has energy budget $B<2n$. The aforementioned algorithm is not sufficient anymore since each route of $r$ can be at most of length $B$.   
Therefore, $r$ needs to return to $S$ to get recharged before the length of the robot's path reaches $B$.
Our proposed algorithm uses the same idea of  performing a DFS traversal of $T_P$ as described in the previous subsection while stopping the DFS traversal process before the route of $r$ has length at most $B$. 
Let $\cR_{\infty}(r)=\{S,v_1,v_2,\ldots,v_l\}$ be the route with respective nodes visited by $r$ while running DFS assuming $B=\infty$. 
Let $\cR_i(r)$ denote a route of $r$ visiting the nodes of $\cR_{\infty}(r)$ when $B<2n$. 
The goal is to obtain $\cR'(r)=(\cR_1(r),\cR_2(r),\ldots,\cR_k(r))$, such that the three conditions listed in Section \ref{section:introduction} are satisfied.

The challenge is to plan each route $\cR_i(r)$ in an online fashion satisfying all three criteria while minimizing both the number of paths $|\cR'(r)|$ and the total length of the paths $l(\cR'(r))$. We use the following approach: 
$\cR_1(r)$ starts from $S$ and visits the nodes of $\cR_{\infty}(r)$ in a sequence. As soon as $\cR_1(r)$ reaches to a node $v_i\in \cR_{\infty}(r)$ such that $\dist(v_i,S) \leq B_{remain}$, it terminates the DFS traversal and returns to $S$. $B_{remain}$ is the energy remained after each move. In route $\cR_2(r)$, $r$ moves to $v_i$ (where it stopped the DFS traversal in $\cR_1(r)$) from $S$ and continues the DFS traversal until it reaches to a node $v_j\in \cR_{\infty}(r)$ from which $\dist(v_j,S) \leq B_{remain}$. Like last route, $r$ then returns to $S$.  
This process then continues until the last node $v_l\in \cR_{\infty}(r)$ is visited in some route $\cR_k(r)$. We later prove that using this approach,
$r$ visits all nodes of $T_P$ 
providing correctness and approximation guarantee claimed in Theorem \ref{theorem:main}. 

\begin{figure}[ht!]
\centering
\includegraphics[width=1.65in]{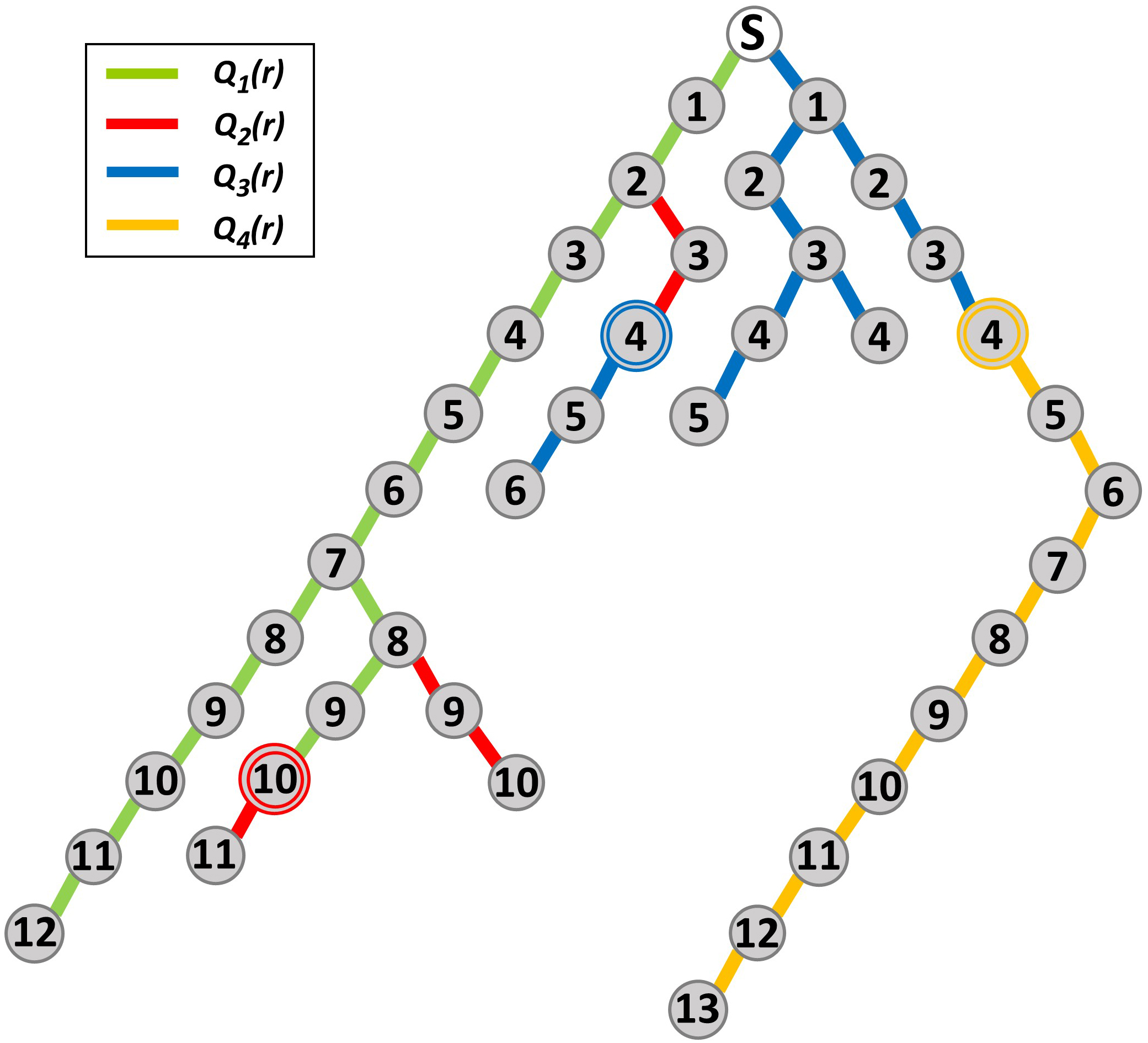}
\includegraphics[width=1.65in]{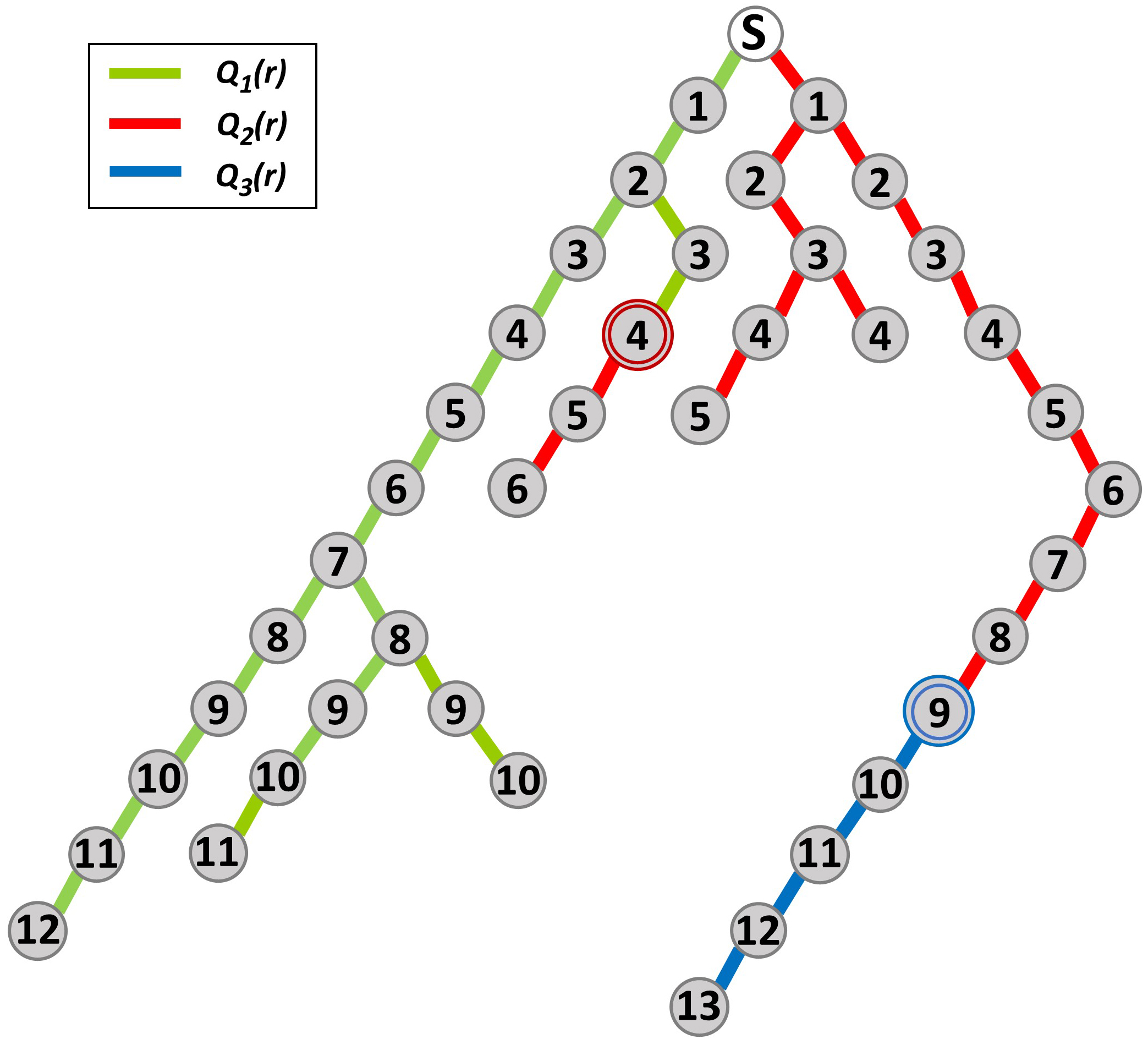}
\vspace{-3mm}
\caption{An illustration of the budgeted DFS traversal by using Algorithm \ref{algorithm:onlinecppalg} when energy budget $B$ is $30$ ({\bf left}) and $40$ ({\bf left}) for the environment $P$ shown in Fig.~\ref{fig:tree}.
For $B=30$, $r$ needs $4$ paths whereas only $3$ paths are needed when $B=40$. The nodes of $T_P$ where $\cR_{i-1}(r)$ stops the DFS traversal are marked by double circles; the next path $\cR_{i}(r)$ continues the DFS traversal from these nodes.   
}
\label{fig:paths}
\end{figure}

We call our algorithm {\cova} (shown in Algorithm \ref{algorithm:onlinecppalg}). Initially, the robot $r$ is at $S$ with the energy budget $B<2n$. 
This is a special situation where $v=S$ and $B_{remain}=B$. 
Robot $r$ then includes the child nodes of $S$ (in contour 1) in $T_P$ making them child nodes of root $S$ in $T_P$ and moves to the leftmost child node, say $v_{1,left}$. The node $v_{1,left}$ is marked visited. 
$B_{remain}$ is decreased by 1 (line $21$).
Robot $r$ then moves from $v_{1,left}$ to the leftmost child node $v_{2,left}$ (in contour 2).  The node $v_{2,left}$ is marked visited and $B_{remain}$ is decreased by 1.  This process continues until at some node $v_i$ (in some contour $i$), $B_{remain}=D_{v_i}$, where $D_{v_i}$ is the distance from $v_i$ to $S$ in $T_P$. 
Robot $r$ then returns to $S$ following the path in $T_P$ (line $23$). After getting fully charged at $S$, $r$ follows the path in $T_P$ to reach $v_i$ to continue the DFS traversal.

\begin{algorithm}[!t]
{\small
{
{\bf Robot:} Initially positioned at the charging station $S$ and it knows its size $L$ and the energy budget $B$. \\
{\bf Environment:} Planar area with radius at most $\lfloor B/2\rfloor$ with center $S$ possibly containing obstacles; obstacle positions and numbers not known. \\
{\bf Data structures:} Tree map $T_P$ and new frontier stack $F$. \\ 
{\bf Initialize:} $T_P =\{S\}$, $F=\{S\}$, 
distance (i.e., depth in tree $T_P$ for a node $v$ of $T_P$) $D_{v}=0$, 
the energy budget remaining $B_{remain}=B$, and node in $T_P$ to continue coverage in the next route $node_{next}=S$.\\
\While{$F\neq \emptyset$} {

\If{$B$ is not even}
{$B_{remain}=B-1$;}
Move to $node_{next}$ from $S$ using the shortest path in $T_P$;\\
$D_v\leftarrow$ the distance from $S$ to $node_{next}$ in $T_P$;\\
$B_{remain}\leftarrow B_{remain}-D_v$;\\
\While{$B_{remain}> D_v$}{
\eIf{$node_{next}$ has unvisited child nodes in $T_P$ or the cell on top of $F$ is the child node of $node_{next}$ in $T_P$} {
\If{the child nodes of $node_{next}$ are not already included in $F$ and $T_P$} {
Include all the child nodes of $node_{next}$ in $T_P$ and $F$ (ordered clockwise from left to right starting from the leftmost child node and insert to $F$ from right to left);\\
}
$v\leftarrow$ the node on the top of $F$ or the leftmost in $T_P$ ($v$ will be the node pushed into $F$ last);\\
Robot $r$ removes $v$ from $F$, moves to $v$, marks $v$ visited in $T_P$;\\
}
{
$v\leftarrow$  the parent node of $node_{next}$ in $T_P$;\\
Robot $r$ moves to $v$;\\
}
$node_{next}\leftarrow v$;\\
$B_{remain}\leftarrow B_{remain}-1$;\\
$D_v\leftarrow$ the distance from $S$ to $node_{next}$;\\
}
Robot $r$ goes to $S$ following a path in $T_P$;\\
$B_{remain}\leftarrow B$ (after $r$ is fully changed) after reaching $S$;\\
}
}
\caption{{\cova}} 
\label{algorithm:onlinecppalg}
}
\end{algorithm}	

At any cell of $P$ (node of $T_P$), if it has a unvisited neighbor cell in the next contour (child node in $T_P$), then $r$ moves to that cell. 
Otherwise, $r$ retreats back to the parent node cell of its current cell in $T_P$ (lines $18-19$). 
During the exploration, anytime $r$ realizes that it has just enough energy remaining $B_{remain}$ to reach back to $S$, it does so by visiting the parent nodes in the tree $T_P$ starting from its current node.

\section{Analysis of the Algorithm}
\label{section:analysis}
In this section, we provide theoretical analysis of {\cova}. We first prove its correctness and then analyze the costs for the performance metrics (1) and (2). 

\noindent{\bf Correctness.}
We start with the following lemma.
\begin{lemma}
\label{lemma:same-order}
Let $\{S,v_1,v_2,\ldots,v_l\}$ be the order of the nodes of $T_P$ (i.e., cells in $P$) visited by the DFS traversal $\cR_{\infty}(r)$ when $B=\infty$. Let $\cR'(r)=\{\cR_1(r),\cR_2(r),\ldots, \cR_k(r)\}$ be the routes of $r$ that collectively visit the nodes of $T_P$  at least once when $B<2n$. The not-yet-visited nodes of $T_P$ (or cells of $P$) are visited in $\cR'(r)$ in the same order as in $\cR_{\infty}(r)$.
\end{lemma}
\begin{proof}
Consider the paths in $\cR'(r)$ when $B=\infty$. In this case, instead of stopping the DFS traversal at some node $\alpha$ and make a round trip to $S$ from $\alpha$, each subsequent paths continue  their traversal without this stoppage. This simulates essentially the behavior of $r$ when $B=\infty$ giving $\cR_{\infty}(r)$ and hence the nodes of $T_P$ (or cells in $P$) are visited in the same order in both (except the nodes visited in the roundtrip to $S$ from the stopped node $\alpha$).   
\end{proof}

\begin{theorem} [{\bf Correctness}]
\label{theorem:correctness}
{\cova} completely covers the environment $P$.
\end{theorem}
\begin{proof}
When $P$ is known a priori, $r$ can visit all the reachable cells in $P$ with an unlimited budget. 
If $P$ is unknown but $B=\infty$, it is also known that through a DFS traversal, each reachable cell of $P$ is guaranteed to be visited where the traversal path is represented by $\cR_{\infty}(r)=\{S,v_1,v_2,\ldots,v_l\}$. 
We have proved in Lemma \ref{lemma:same-order} that when $P$ is unknown but $B < 2n$, the nodes of $T_P$ are visited in the same order. This immediately provides the guarantee that all reachable cells of $P$ will be visited by $r$. Hence, proved.
\end{proof}

\vspace{1mm}
\noindent{\bf Approximation Ratio.} We prove the following theorem. 

\begin{theorem}[{\bf Approximation}]
\label{theorem:approximation}
{\cova} achieves $10$-approximation for both performance metrics -- the number of paths and the total lengths of the paths.
\end{theorem}
\begin{proof}
Let $T$ be a tree of depth at most $\lfloor B/2\rfloor$. Let $r$ be a robot with energy budget at least $B$. 
After $r$ starts from $S$ to visit the nodes of $T$, due to the limited energy budget, $r$ may need to stop the coverage of $T$ and visit the charging station $S$ again before rest of the nodes in $T$ can be covered. 

Let $OPT$ be the DFS exploration strategy for $r$ that consists of the minimum number of routes, i.e., the minimum number of times $r$ needs to visit $S$ before $T$ is completely covered. 
Let $ALG$ be the DFS exploration strategy 
that visits the nodes of $T$ (starting from $S$) using a DFS traversal where length of each route is bounded by $B$. As soon as the battery is fully charged at $S$, in the next route, $r$ directly goes to the node of $T$ where it stopped the DFS traversal in the last route and continues covering the unvisited nodes of $T$.   
For any tree $T$ of depth $\lfloor B/2\rfloor$ and any robot $r$ of energy budget at least $B$, 
we have the following result from \cite{Das2018} on the number of routes $|\cR_{ALG}(r)|$, of the strategy $ALG$, compared to the number of routes $|\cR_{OPT}(r)|$, of strategy $OPT$: $|\cR_{ALG}(r)|\leq 10\cdot |\cR_{OPT}(r)|$. 
Moreover, let $l(\cR_{ALG}(r))$ be the total length traversed by $r$ while using the strategy $ALG$. Let $l(OPT(r))$ be the optimal length traversed by $r$. 
Again from \cite{Das2018}, we have that $l(\cR_{ALG}(r)) \leq 10\cdot l(OPT(r))$.

The above results are interesting meaning that the bounds hold for any arbitrary DFS traversal $\cR(r)$ of $T$ by $r$. That means that the whole DFS traversal $\cR(r)$ does not need to be known beforehand (i.e., can be computed online not knowing $T$ in advance). Moreover, each route can be constructed without any knowledge on the yet unvisited part of $T$. 

We now discuss how the two results can be adapted to prove the same bounds for {\cova}. Consider a DFS traversal $\cR_\infty(r)$ of $P$ (or equivalently $T_P$) by $r$ when $B=\infty$. We have from Lemma \ref{lemma:same-order} that the routes in $\cR'(r)$ visit the not-yet-visited nodes of $P$ in the order same as in $\cR(r)$. Moreover, the tree map $T_P$ of $P$ formed during the exploration is of depth at most $\lfloor B/2\rfloor$. Therefore, $|\cR'(r)|\leq 10\cdot |\cR_{OPT}(r)|$ and $l(\cR'(r))\leq 10\cdot l(OPT(r))$.   
\end{proof}

\vspace{1mm}
\noindent{\bf Proof of Theorem \ref{theorem:main}:} Theorems \ref{theorem:correctness} and  \ref{theorem:approximation} prove Theorem \ref{theorem:main} for $L=1$. Since the cells are decomposed proportional to the robot size $L\times L$, a simple adaptation of the analysis again gives $10$-approximation for {\cova} for $L>1$. The correctness analysis remains unchanged. \qed

\begin{figure*}[!t]
\centering
\begin{tabular}{ccccc}
\hspace{-0.2in}
\includegraphics[width=0.2\linewidth]{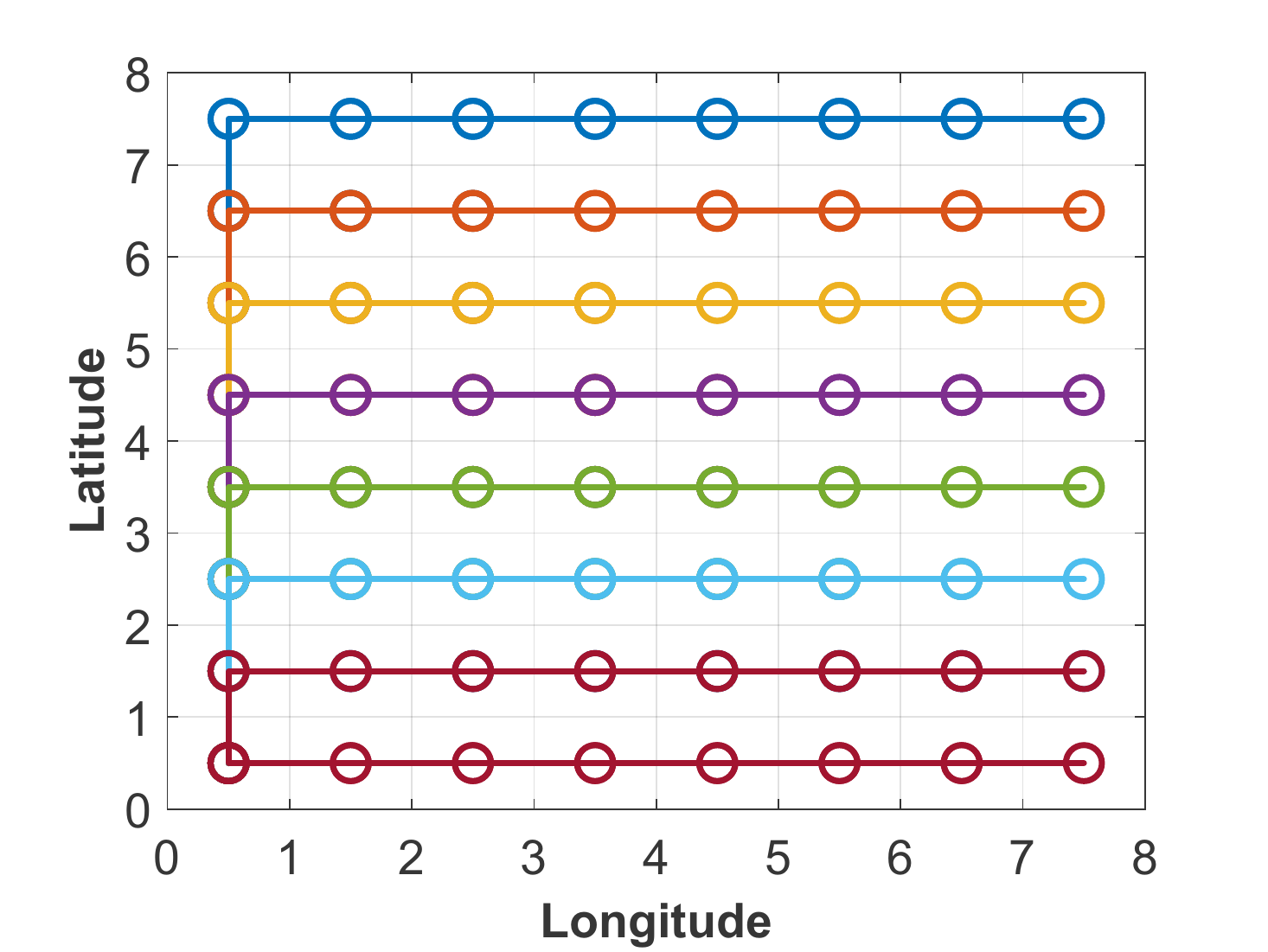}&
\hspace{-0.2in}\includegraphics[width=0.2\linewidth]{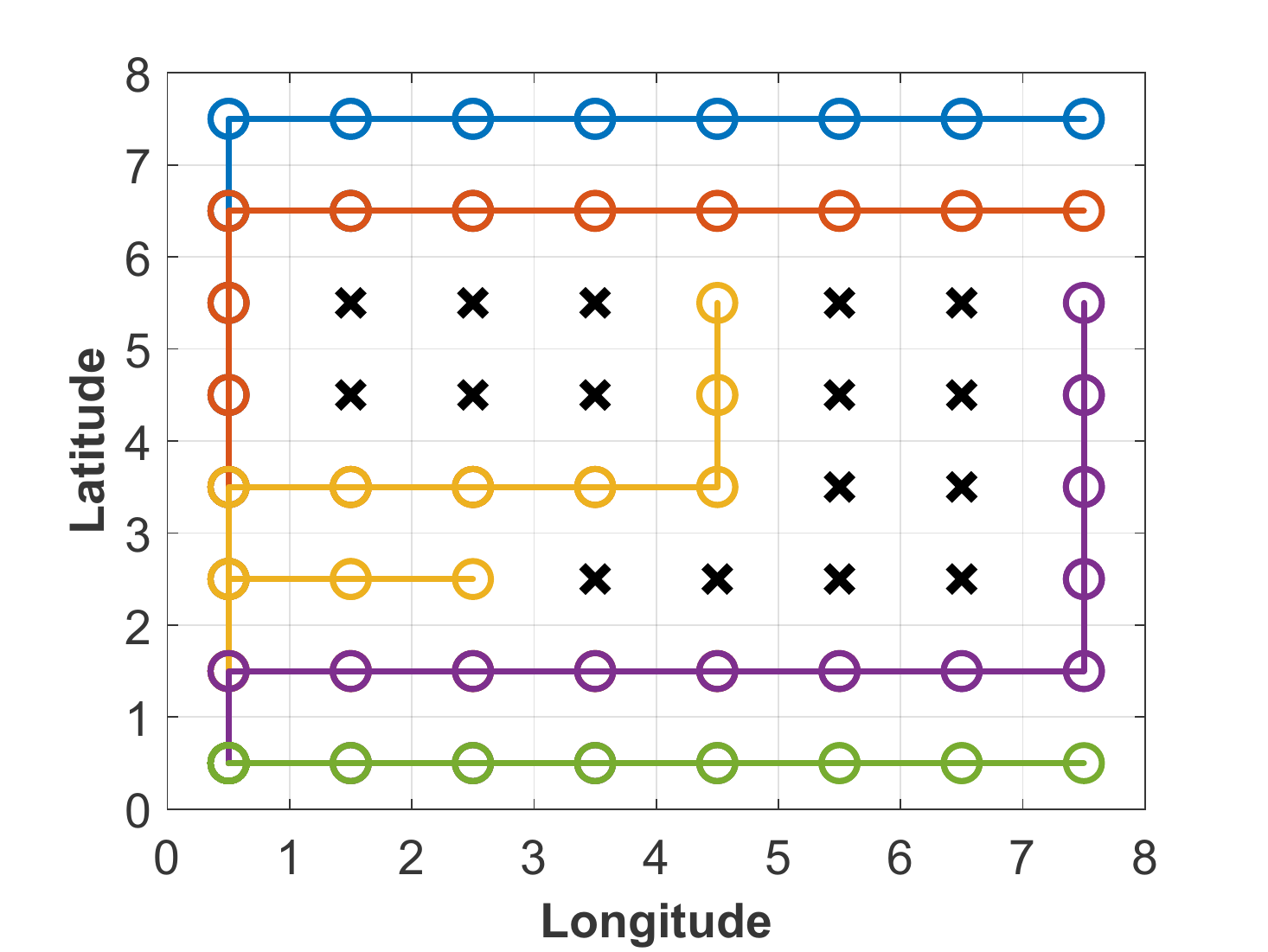}&
\hspace{-0.2in}\includegraphics[width=0.2\linewidth]{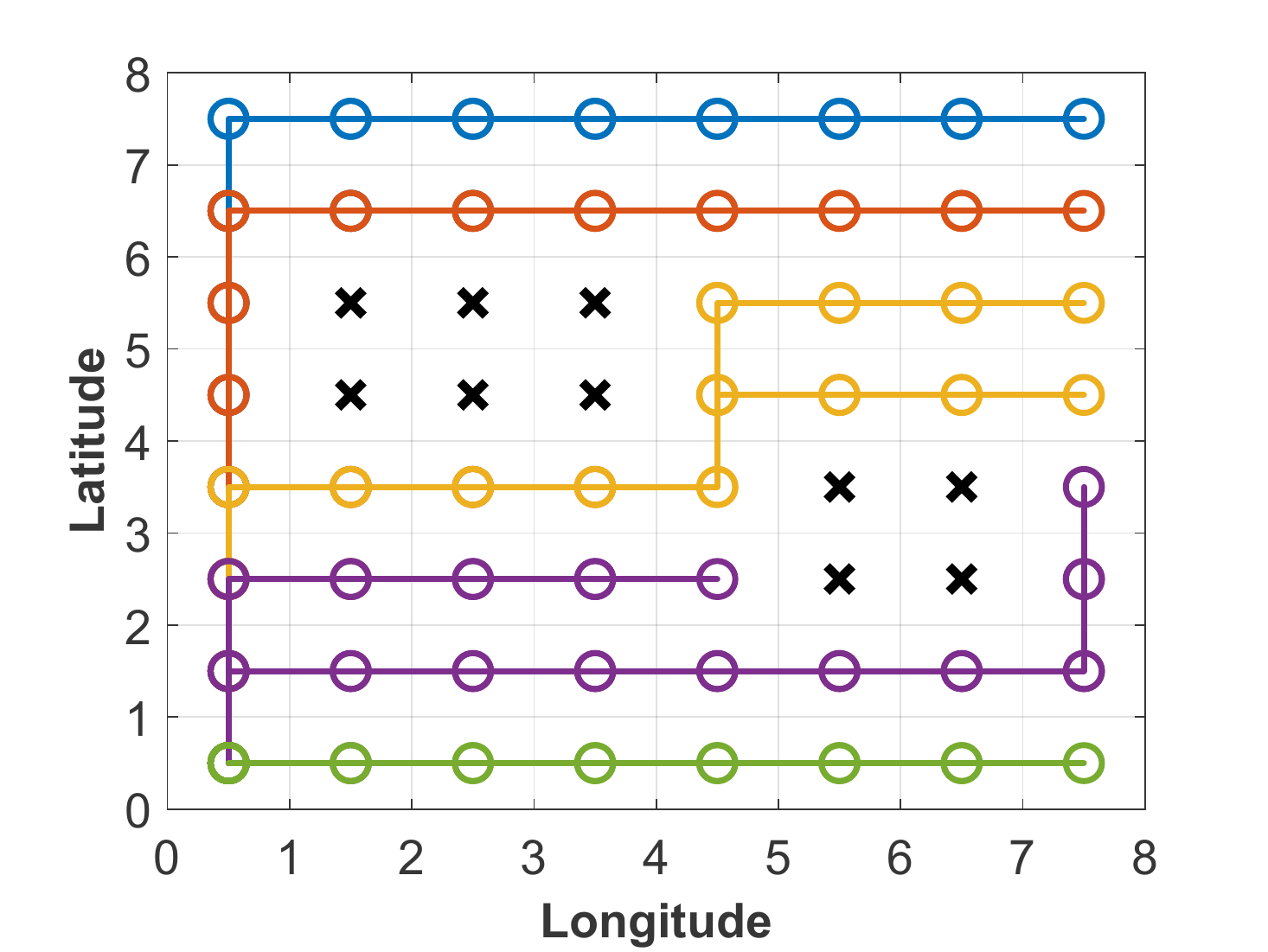}&
\hspace{-0.2in}\includegraphics[width=0.2\linewidth]{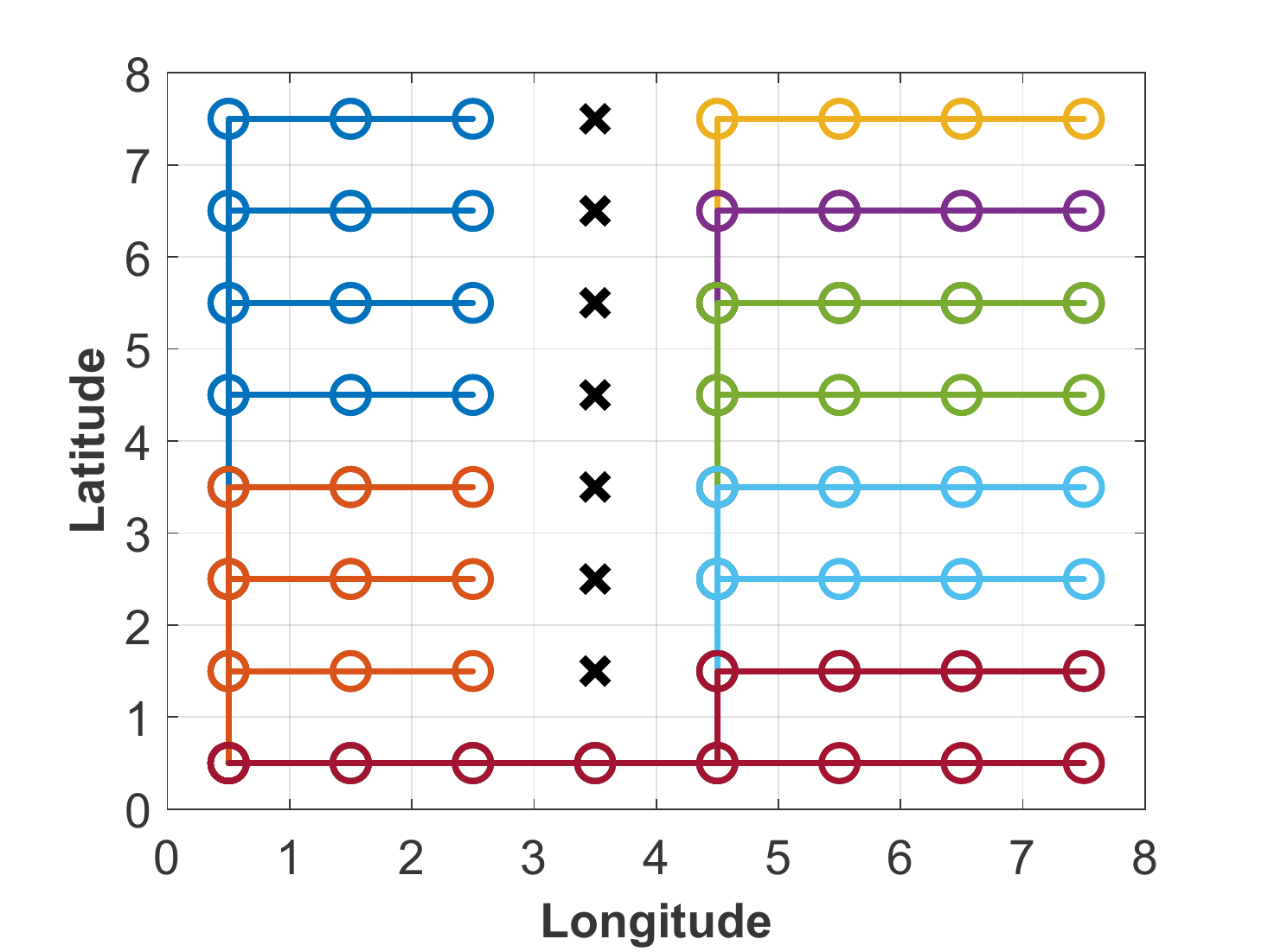}&
\hspace{-0.2in}\includegraphics[width=0.2\linewidth]{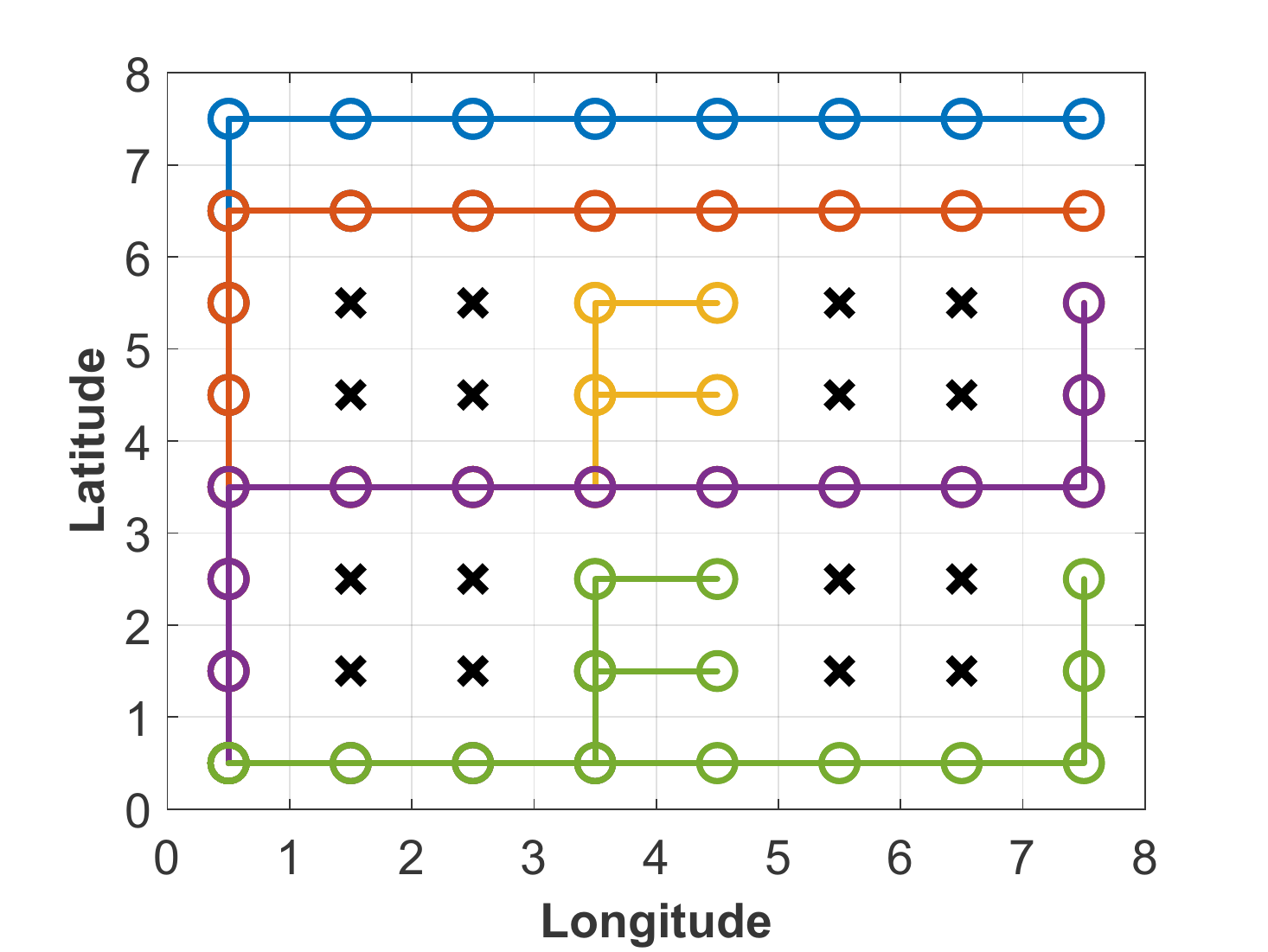}\\
\end{tabular}
\vspace{-2mm}
\caption{An illustration of five different configurations (Conf1 to Conf5) and the paths followed by $r$ using {\cova}. The obstacles are represented with `x'. $S$ is in the bottom-left corner cell. When plotted, later paths have been given higher priority and they are shown in the foreground.
}
\label{fig:show_paths}
\end{figure*}

\section{Evaluation}
\label{section:simulation}
\noindent{\bf Settings.}
We have implemented the proposed algorithm {\cova} using Java programming language on a desktop computer with an Intel i7-7700 CPU and 16GB RAM. The robot size $L$ is set to $1$. We have created five different environments with both convex and concave obstacles in them. The test environments are of the same dimension -- $8 \times 8$ $(l=8)$. The budget $B$ is set to $4l$ (unless otherwise mentioned), i.e., four times the size of each side of the environment. Note that this is the lowest possible budget to completely cover the environment. The charging station $S$ is placed at the left-bottom corner in every test environment. These environments are shown in Fig.~\ref{fig:show_paths} (Conf1 to Conf5). 

Empirically, we have mainly focused on three metrics to evaluate the quality of the proposed algorithm: 1) time to cover the environment, 2) total path length traversed by the robot, and 3) approximation ratio. We also compare our results against the current state-of-the-art algorithm that solves {\cov} under energy constraint \cite{Sharma2019}. 

\noindent{\bf Results.}
First we empirically verify the theoretically-proved constant-factor approximation bound. Let $n$ denote the number of reachable cells in the environment. Then $MIN = \frac{2n}{B}$ will indicate the {\em absolute minimum} number of paths required by the robot to completely cover the environment \cite{Sharma2019, WeiI18}. No optimal DFS strategy ($OPT$) can guarantee a better approximation bound than $MIN$. Here, we compare our experimental result against $MIN$ as a comparison against any $OPT$ cannot make our empirical approximation bound worse. 
The result is shown in Fig. \ref{fig:approx}(a). The state-of-the-art bound of $\log(B)$ (when $L=1$) is also plotted for reference. The figure shows that in practice, the approximation bound is well below the $10$-factor theoretical worst-case bound. Also, in all of the test environments, our proposed algorithm outperforms the state-of-the-art  $\log(B)$-approximation \cite{Sharma2019}. 

\begin{figure}
\begin{center}
\begin{tabular}{cc}
\hspace{-0.15in}\includegraphics[width=0.5\linewidth]{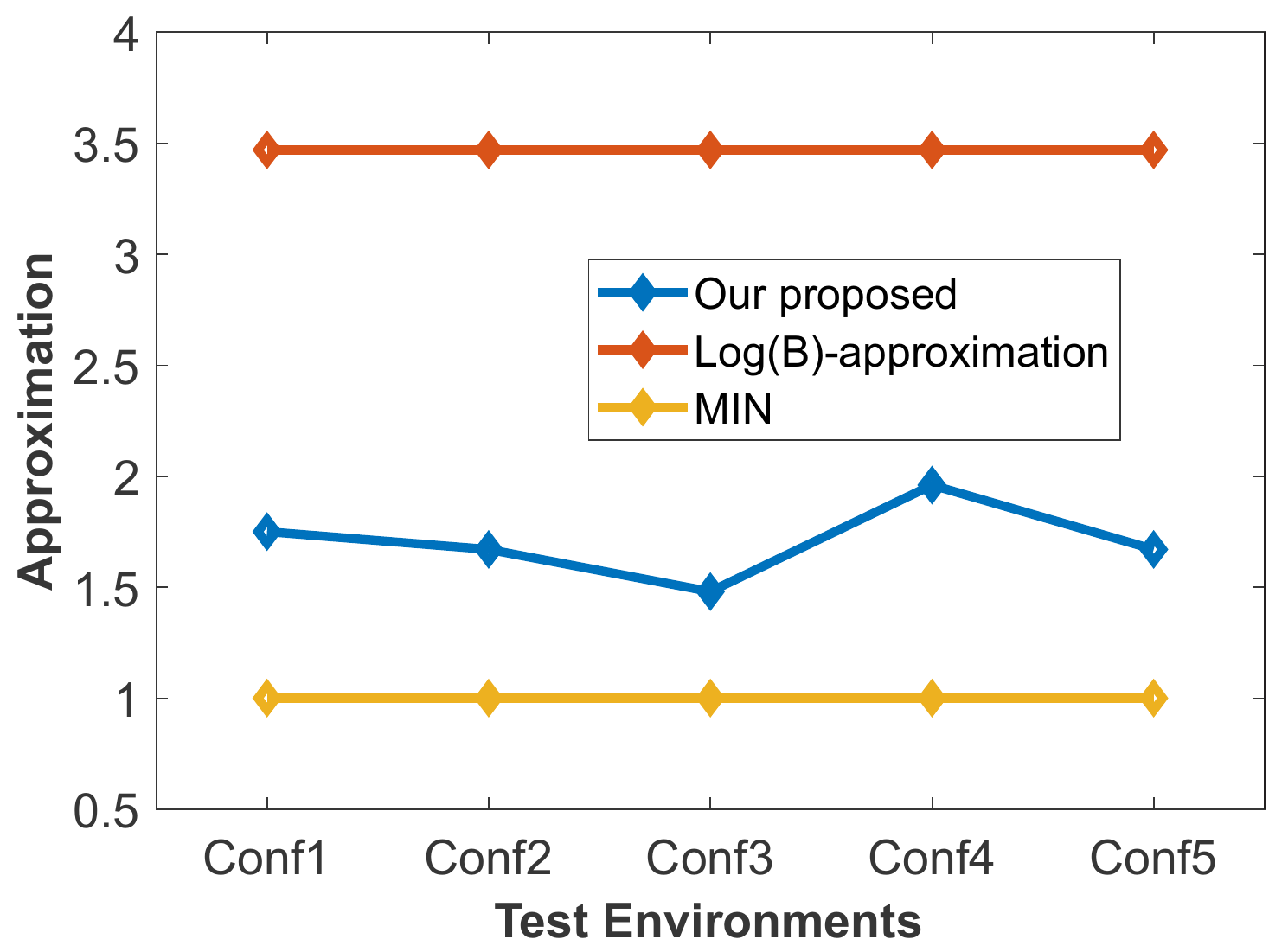}&
\includegraphics[width=0.5\linewidth]{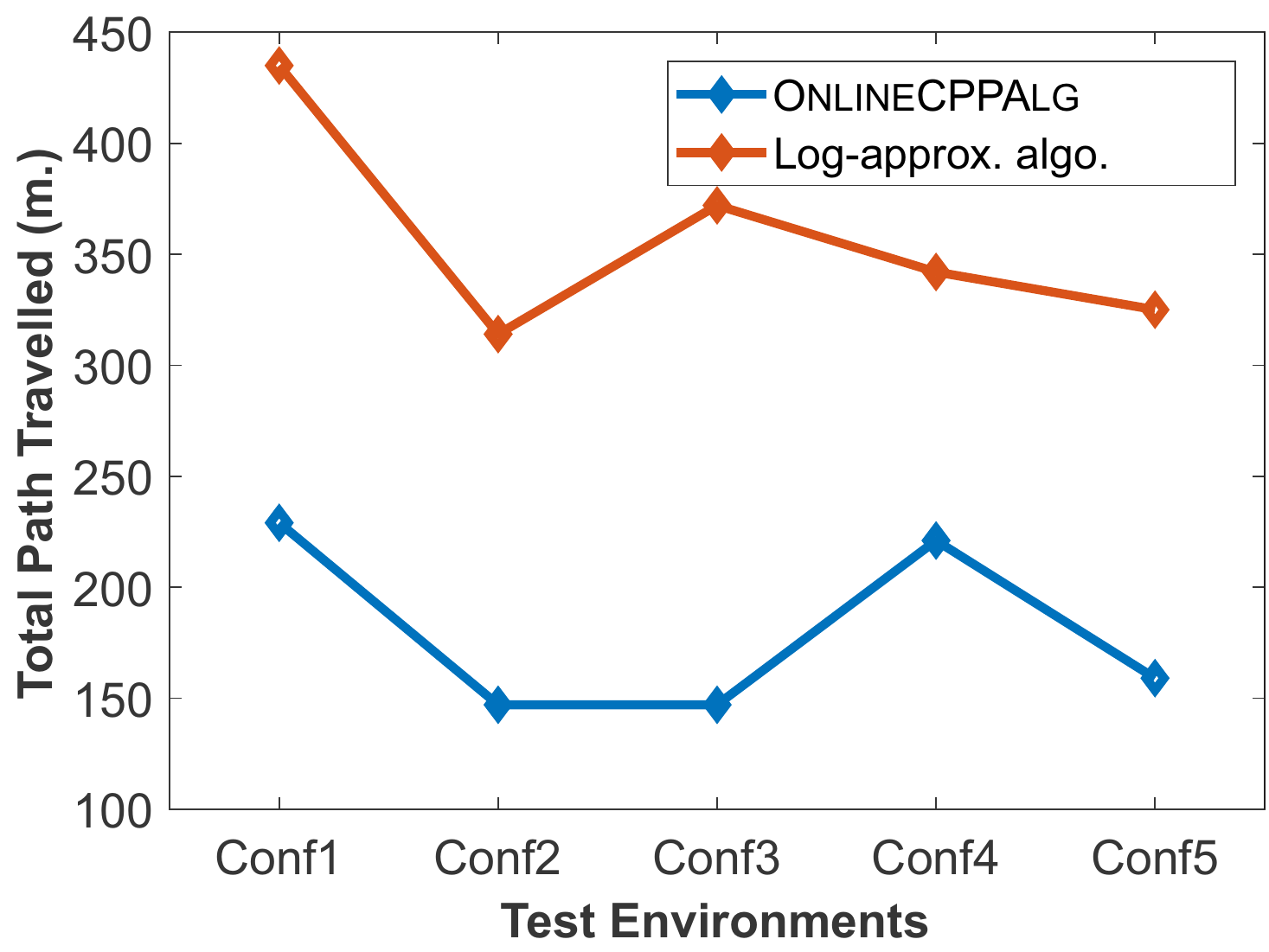}\\
(a)&(b)\\
\end{tabular}
\end{center}
\vspace{-4mm}
\caption{a) Empirical validation of the constant-factor approximation on different environment configurations; b) Comparison of path lengths travelled by the robot using our algorithm and the algorithm in \cite{Sharma2019}.}
\label{fig:approx}
\end{figure}

As we are minimizing the total path length travelled by the robot to completely cover the environment, we are interested to compare this metric for our algorithm against the algorithm in \cite{Sharma2019}. The result is shown in Fig. \ref{fig:approx}(b). It can be clearly observed from this plot, that our proposed algorithm outperforms the algorithm in \cite{Sharma2019} in terms of the travelled path length -- by an average ratio of $2.03$ while the maximum ratio is $2.53$ (Conf3). 

\begin{figure}
\begin{center}
\begin{tabular}{cc}
\hspace{-0.15in}\includegraphics[width=0.55\linewidth]{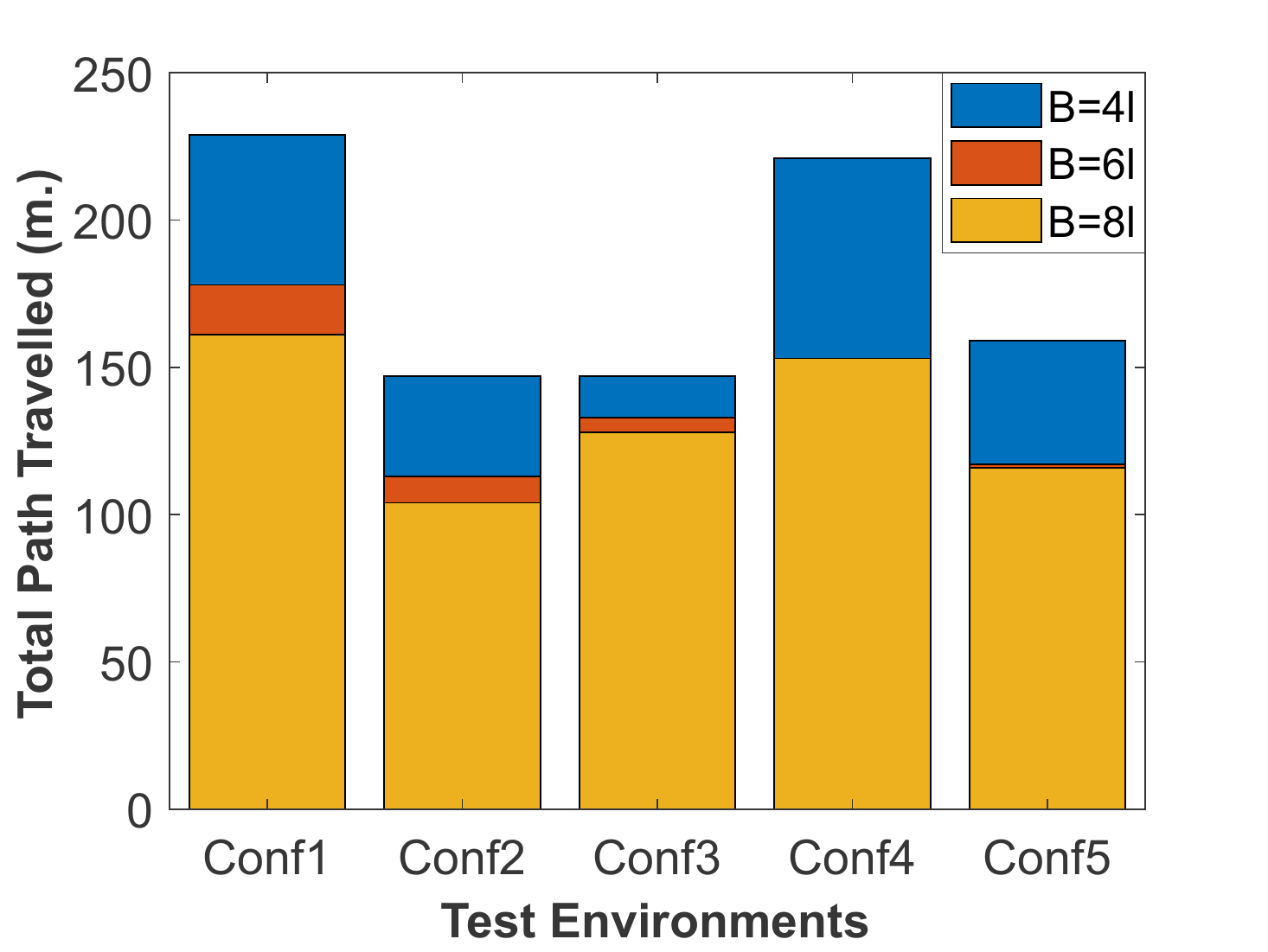}&
\hspace{-0.2in}\includegraphics[width=0.55\linewidth]{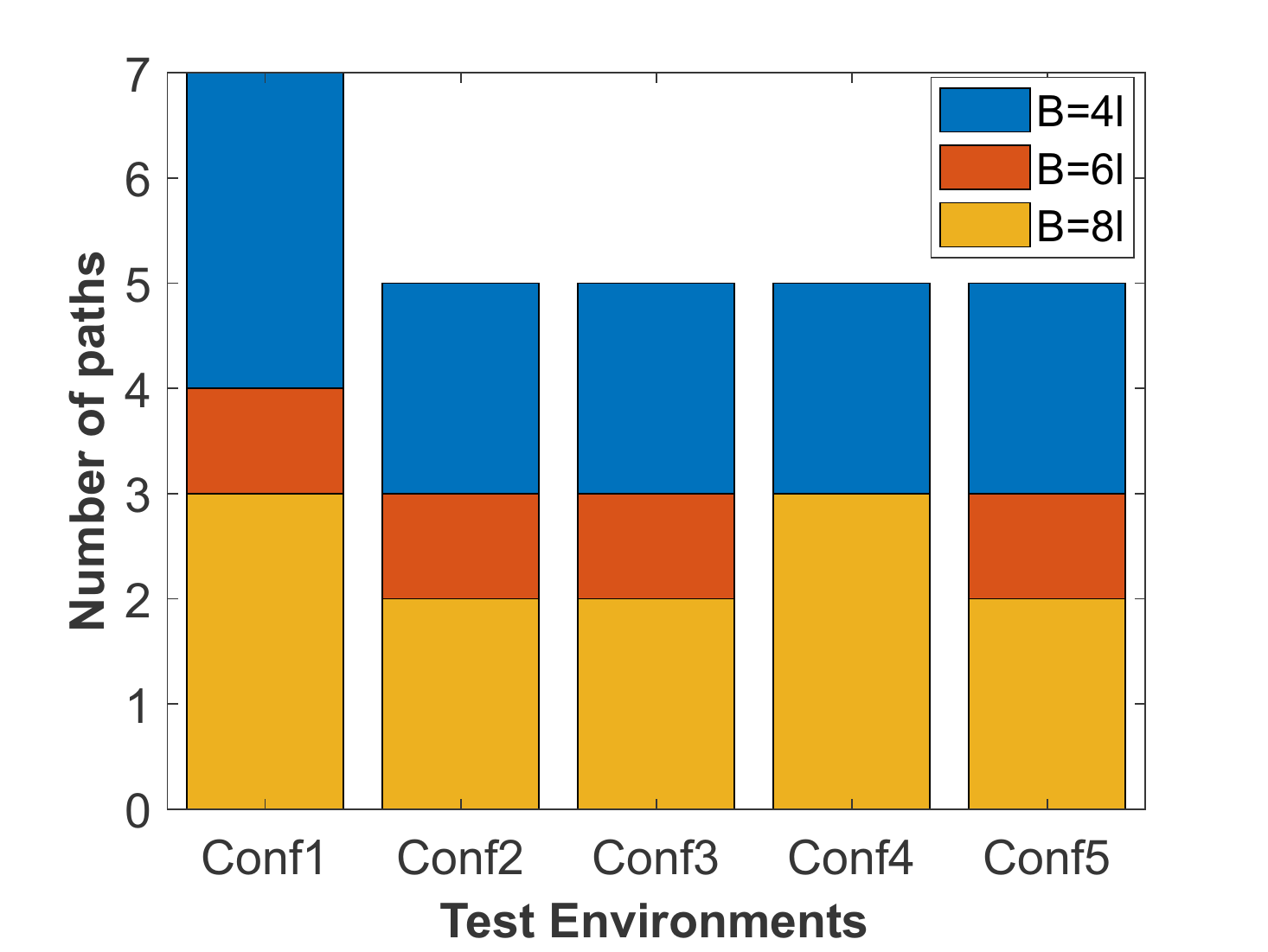}\\
(a)&(b)\\
\end{tabular}
\end{center}
\vspace{-4mm}
\caption{Varying the budget: a) Comparison of travelled path lengths; b) Comparison of total number of paths (i.e., number of visits to $S$).}
\label{fig:varyB}
\end{figure}

Next we are interested to investigate the effect of changing the budget amount on the travelled path length. In order to do this, we vary $B$ between $\{4l,6l,8l\}$. The result is shown in Fig. \ref{fig:varyB}(a). With higher budget, $r$ could cover more cells in one path than with a lower budget. This fact is also reflected in the plot where travelled path length is higher with lower budget and vice-versa. Similarly, when the budget is higher and $r$ is covering more cells in a single path, it needs to come back to the charging station less often and consequently, the total number of paths also reduces. This can be observed in Fig. \ref{fig:varyB}(b). On average, $r$ visited $S$ $2.30$ times more with $B=4l$ than with $B=8l$.

\begin{figure}[ht!]
\begin{center}
\includegraphics[width=0.65\linewidth]{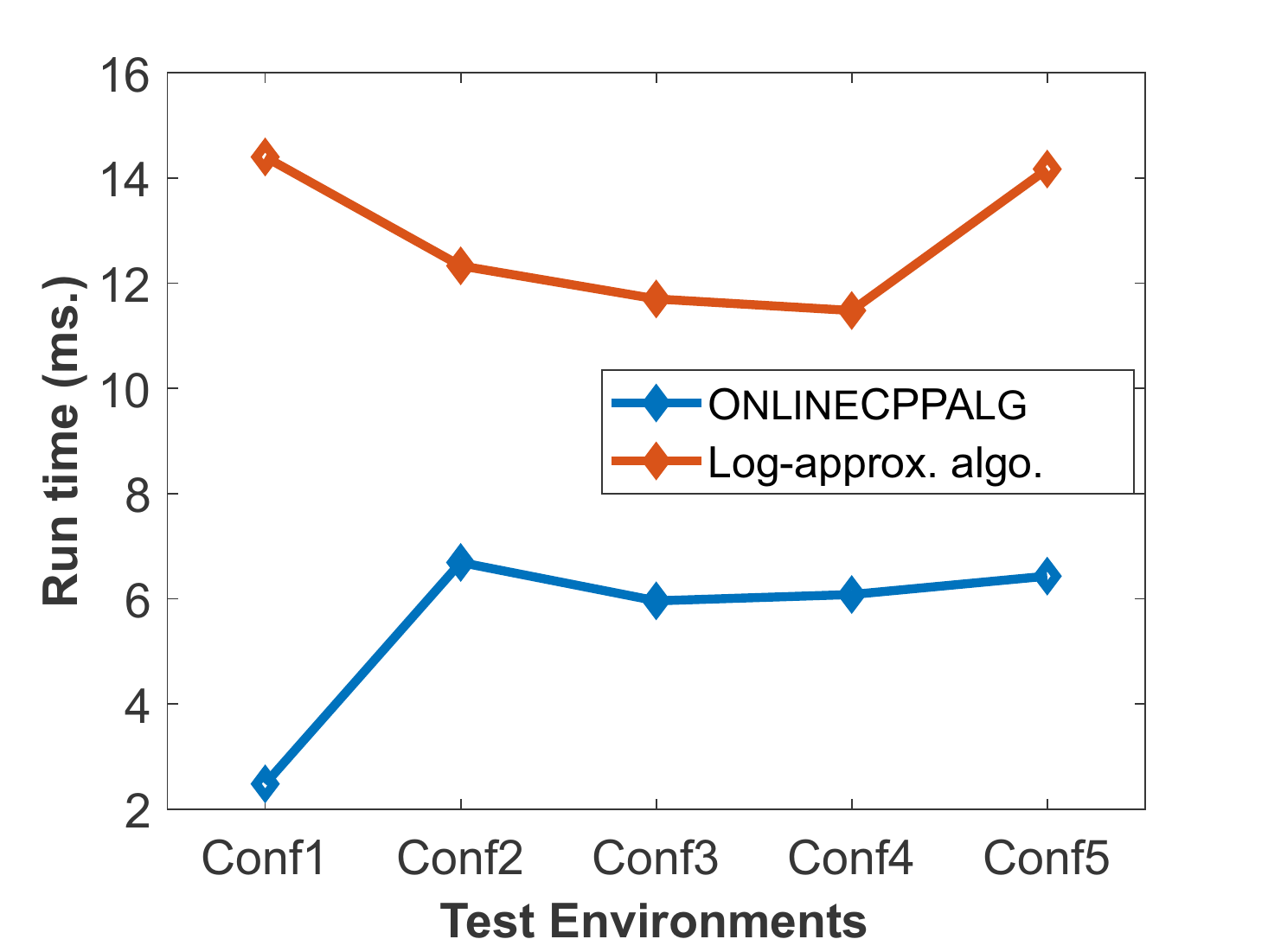}
\end{center}
\vspace{-4mm}
\caption{Runtime comparison against the log-approximation algorithm \cite{Sharma2019}.}
\label{fig:time_compare_aamas}
\end{figure}

Next we are interested to investigate the run time of the proposed algorithm. We also compare this metric against the algorithm proposed in \cite{Sharma2019}. The result is shown in Fig. \ref{fig:time_compare_aamas}. On average, our algorithm is shown to be $2.74$ times faster than \cite{Sharma2019} while the maximum ratio is $5.80$ (Conf1). 
Finally, the paths followed by $r$ in different environment configurations are shown in Fig. \ref{fig:show_paths}; a video of the simulation is also submitted.

\section{Conclusion and Future Work}
We have presented an algorithm for {\cov} by an energy-constrained robot achieving $10$-approximation, improving significantly on the state-of-the-art $O(\log(B/L))$-approximation \cite{Sharma2019}. 
It is also simpler to implement compared to \cite{Sharma2019}. Our simulation results validate the approximation bound established theoretically. We have empirically shown that our proposed approach outperforms the state-of-the-art algorithm both in terms of run time and total traversal cost for the complete coverage. In the future, we plan to test our algorithm in a real-world setting.

{\small
\bibliographystyle{abbrv}
\bibliography{references}
}

\end{document}